\documentclass{article}

\usepackage{amsthm, amsmath, amsfonts}
\usepackage{bm}
\usepackage{tikz,pgfplots}
\pgfplotsset{width=10cm,compat=1.9}

\newtheorem{prop}{Proposition}
\newtheorem{theor}{Theorem}
\newtheorem{corol}{Corollary}
\newtheorem{lemma}{Lemma}
\newtheorem{defin}{Definition}

\usepackage{hyperref}

\usepackage[accepted]{mystyle}

\begin{document}

\twocolumn[
\icmltitle{Elementary superexpressive activations}

\begin{icmlauthorlist}
\icmlauthor{Dmitry Yarotsky}{sk}
\end{icmlauthorlist}
\icmlaffiliation{sk}{Skolkovo Institute of Science and Technology, Moscow, Russia. E-mail: \texttt{d.yarotsky@skoltech.ru}}

\vskip 0.3in
]
\printAffiliationsAndNotice{}

\begin{abstract} We call a finite family of activation functions \emph{superexpressive} if any multivariate continuous function can be approximated by a neural network that uses these activations and has a fixed architecture only depending on the number of input variables (i.e., to achieve any accuracy we only need to adjust the weights, without increasing the number of neurons). Previously, it was known that superexpressive activations exist, but their form was quite complex. We give examples of very simple superexpressive families: for example, we prove that the family $\{\sin, \arcsin\}$ is superexpressive. We also show that most practical activations (not involving periodic functions) are not superexpressive. 
\end{abstract}

\section{Introduction}
In the study of approximations by neural networks, an interesting fact is the existence of activation functions that allow to approximate any continuous function on a given compact domain with arbitrary accuracy by using a network with a finite, fixed architecture independent of the function and the accuracy (i.e., merely by adjusting the network weights, without increasing the number of neurons). We will refer to this property as ``superexpressiveness''. The existence of superexpressive activations can be seen as a consequence of a result of \cite{maiorov1999lower}. 
\begin{theor}[\citealt{maiorov1999lower}]\label{th:maiorov}
There exists an activation function $\sigma$ which is real analytic, strictly increasing, sigmoidal (i.e., $\lim_{x\to -\infty}\sigma(x)=0$ and $\lim_{x\to +\infty}\sigma(x)=1$), and such that any $f\in C([0,1]^d)$ can be uniformly approximated with any accuracy by expressions $\sum_{i=1}^{6d+3}d_i\sigma(\sum_{j=1}^{3d}c_{ij}\sigma(\sum_{k=1}^d w_{ijk}x_k+\theta_{ij})+\gamma_i)$ with some parameters $d_i,c_{ij},w_{ijk},\theta_{ij},\gamma_i$. 
\end{theor}
The proof of this theorem includes two essential steps. In the first step, the result is proved for univariate functions (i.e., for $d=1$). The key idea here is to use the separability of the space $C([0,1])$, and construct a (quite complicated) activation by joining all the functions from some dense countable subset.  In the second step, one reduces the multivariate case to the univariate one by using the Kolmogorov Superposition Theorem (KST).

Though the superexpressive activations constructed in the proof of Theorem \ref{th:maiorov} have the nice properties of analyticity, monotonicity and boundedness, they are nevertheless quite complex and non-elementary -- at least, not known to be representable in terms of finitely many elementary functions. See the papers \cite{ismailov2014approximation, guliyev2016single, guliyev2018approximation1, guliyev2018approximation2} for refinements and algorithmic aspects of such and similar activations, as well as the papers \cite{kuurkova1991kolmogorov, kuurkova1992kolmogorov, igelnik2003kolmogorov, montanelli2020error, schmidt2020kolmogorov} for further connections between KST and neural networks.  

There is, however, another line of research in which some weaker forms of superexpressiveness have been recently established for  elementary (or otherwise simple) activations. The weaker form means that the network must grow to achieve higher accuracy, but this growth is much slower than the power laws expected from the abstract approximation theory under standard regularity assumptions \cite{devore1989optimal}. In particular, results of \cite{yarotsky2019phase} imply that a deep network having both $\sin$ and ReLU activations can approximate Lipschitz functions with error $O(e^{-cW^{1/2}})$, where $c>0$ is a constant and $W$ is the number of weights. Results of \cite{shen2020neural} (see also \cite{shen2020deep}) imply that a three-layer network using the floor $\lfloor\cdot\rfloor$, the exponential $2^x$ and the step function $\mathbf 1_{x\ge 0}$ as activations can approximate Lipschitz functions with an exponentially small error $O(e^{-cW})$.

In the present paper, we show that there are activations superexpressive in the initially mentioned strong sense and yet constructed using simple elementary functions; see Section \ref{sec:super}. For example, we prove that there are fixed-size networks with the activations $\sin$ and $\arcsin$ that can approximate any continuous function with any accuracy. On the other hand, we show in Section \ref{sec:absence} that most practically used activations (not involving periodic functions) are not superexpressive.           

\section{Elementary superexpressive families}\label{sec:super}

\begin{figure}
\begin{center}
\begin{tikzpicture}[scale=1.25]
\node [draw, red, fill=red, rotate=45] (In1) at (0,-0.7) {};
\node [draw, red, fill=red, rotate=45] (In2) at (0,0) {};
\node [draw, red, fill=red, rotate=45] (In3) at (0,0.7) {};

\node [draw, circle, blue, fill=blue] (H1) at (1.1,-0.9) {};
\node [draw, circle, blue, fill=blue] (H2) at (1,-0.4) {};
\node [draw, circle, blue, fill=blue] (H3) at (0.8, 0.3) {};
\node [draw, circle, blue, fill=blue] (H4) at (1, 0.9) {};

\node [draw, circle, blue, fill=blue] (H5) at (2,-0.7) {};
\node [draw, circle, blue, fill=blue] (H6) at (2.1,-0.1) {};
\node [draw, circle, blue, fill=blue] (H7) at (1.9, 0.7) {};

\node [draw, black!40!green, fill=black!40!green] (Out) at (3,0) {};

\draw[thick,->] (In1) edge (H1) (H1) edge (H5) (H5) edge (Out) (In1) edge (H2) (H2) edge (H6) (H6) edge (Out) (In1) edge (H3) (H3) edge (H6) (In2) edge (H2) (In2) edge (H3) (In3) edge (H4) (H4) edge (H7) (H7) edge (Out) (In3) edge (Out) (H1) edge (H7) (In1) edge (H7);

\node at (-0.3, 0.3) {in};
\node at (3.2, 0.3) {out};
\end{tikzpicture}
\caption{An example of network architecture with 3 input neurons, 1 output neuron and 7 hidden neurons.}\label{fig:net}
\end{center}
\end{figure}
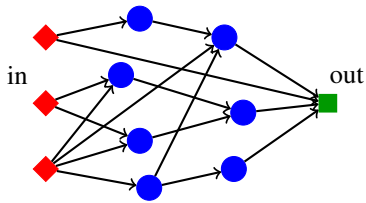

Throughout the paper, we consider standard feedforward neural networks. The architecture of the network is defined by a directed acyclic graph connecting the neurons (see Fig.~\ref{fig:net}). A network implementing a scalar $d$-variable function has $d$ input neurons, one output neuron and a number of hidden neurons. A hidden neuron computes the value $\sigma(\sum_{i=1}^nw_iz_i+h),$ where $w_i$ and $h$ are the weights associated with this neuron, $z_i$ are the incoming connections from other hidden or input neurons, and $\sigma$ is an activation function. We will generally allow different hidden neurons to have different activation functions. The output neuron computes the value $\sum_{i=1}^nw_iz_i+h$ without an activation function. 

Some of our activations (in particular, $\arcsin$) are naturally defined only on a subset of $\mathbb R$. In this case we ensure that the inputs of these activations always belong to this subset. 

Throughout the paper, we consider approximations of functions $f\in C([0,1]^d)$ in the uniform norm $\|\cdot\|_\infty$. We generally denote vectors by boldface letters; the components of a vector $\mathbf x$ are denoted $x_1,x_2,\ldots$.

We give now the key definition of the paper.  
\begin{defin}
We call a finite family $\mathcal A$ of univariate activation functions  \emph{superexpressive for dimension $d$} if there exists a fixed $d$-input network architecture with each hidden neuron equipped with some fixed activation from the family $\mathcal A$, so that any function $f\in C([0,1]^d)$ can be approximated on $[0,1]^d$ with any accuracy in the uniform norm $\|\cdot\|_\infty$ by such a network, by adjusting the network weights. We call a family $\mathcal A$ simply \emph{superexpressive} if it is superexpressive for all $d=1,2,\ldots$ We refer to respective architectures as \emph{superexpressive for $\mathcal A$}.
\end{defin}

Recall that the Kolmogorov Superposition Theorem (KST) \cite{kolmogorov1957representation} proves that any multivariate continuous function can be expressed via additions and univariate continuous functions. The following version of this theorem is taken from \cite{maiorov1999lower}.  

\begin{theor}[KST] There exist $d$ constants $\lambda_j>0, j=1,\ldots,d, \sum_{j=1}^d \lambda_j\le 1$, and $2d+1$ continuous strictly increasing functions $\chi_i,i=1,\ldots,2d+1,$ which map $[0,1]$ to itself, such that every $f\in C([0,1]^d)$ can be represented in the form 
$$f(x_1,\ldots,x_d)=\sum_{i=1}^{2d+1}g\Big(\sum_{j=1}^d\lambda_j\chi_i(x_j)\Big)$$ 
for some $g\in C([0,1])$ depending on $f.$
\end{theor}

An immediate corollary of this theorem is a reduction of multivariate superexpressiveness to the univariate one. 

\begin{corol}\label{corol:1}
If a family $\mathcal A$ is superexpressive for dimension $d=1$, then it is superexpressive for all $d$. Moreover, the number of neurons and connections in the respective superexpressive architectures scales as $O(d^2)$.
\end{corol}
The proof follows simply by approximating the functions $\chi_i$ and $g$ in the KST by univariate superexpressive networks.

Our main result establishes existence of simple superexpressive families constructed from finitely many elementary functions. The full list of properties of the activations that we use is relatively cumbersome, so we find it more convenient to just prove the result for a few particular examples rather than attempt to state it in a general form.

\begin{theor}\label{th:main}
Each of the following families of activation functions is superexpressive:
\begin{align*}
\mathcal A_1={}&\{\sigma_1, \lfloor\cdot\rfloor\},\\
\mathcal A_2={}&\{\sin,\arcsin\},\\
\mathcal A_3={}&\{\sigma_3\},
\end{align*}
where $\sigma_1$ is any function that is real analytic and non-polynomial in some interval $(\alpha,\beta)\subset\mathbb R$, and 
$$\sigma_3(x)=\begin{cases}-\tfrac{1}{x}, & x<-1,\\
\tfrac{1}{\pi}(x\arcsin x+\sqrt{1-x^2})+\tfrac{3}{2}x, & x\in[-1,1],\\
7-\tfrac{3}{x}+\tfrac{\sin x}{\pi x^2}, & x>1.
\end{cases}
$$
The function $\sigma_3$ is $C^1(\mathbb R)$, bounded, and strictly monotone increasing.
\end{theor}

The family $\mathcal A_1$ is a generalization of the family for which \cite{shen2020neural} proved a weaker superexpressiveness property. 

The function $\sigma_3$ is given as an example of an explicit superexpressive activation that is smooth and sigmoidal (see Fig.~\ref{fig:sigma3}).

\begin{figure}
\begin{center}
\begin{tikzpicture}
\begin{axis}[scale=0.7, xlabel=$x$, ylabel=$\sigma_3(x)$]
\addplot[domain=-7:-1, 
         samples=200, 
         color=blue, 
         thick]{-1/x)};
\addplot[domain=-1:1, 
         samples=100, 
         color=blue, 
         thick]{1/pi*(x*pi/180*asin(x)+sqrt(1-x*x))+1.5*x+2};
\addplot[domain=1:7, 
         samples=200, 
         color=blue, 
         thick]{7-3/x+(1/pi)*sin(180*x)/(x*x)};
\end{axis}
\end{tikzpicture}
\caption{The function $\sigma_3$ from the statement of Theorem~\ref{th:main}.}\label{fig:sigma3}
\end{center}
\end{figure}
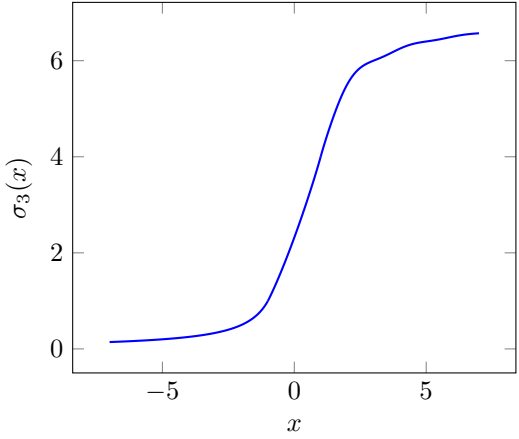

\begin{proof}[Proof of Theorem \ref{th:main}.] We consider the families $\mathcal A_1,\mathcal A_2,\mathcal A_3$ one by one. 

\medskip
\noindent
\emph{Proof for $\mathcal A_1$.} Given a function $f\in C([0,1]^d)$, we will construct the approximation $\widetilde f$ as a function piecewise constant on a partition of the cube $[0,1]$ into a grid of smaller cubes. Following the paper \cite{shen2020neural}, we specify these cubes by mapping them to integers with the help of function $\lfloor\cdot\rfloor$. Specifically, take some $M\in\mathbb N$ and let 
\begin{equation}\label{eq:gm}g_M(x_1,\ldots,x_d)=1+\sum_{k=1}^d(M+1)^{k-1}\lfloor Mx_k \rfloor.\end{equation}
The function $g_M$ is integer-valued and constant on the cubes $I_{M,\mathbf m}=[\tfrac{m_1}{M},\tfrac{m_1+1}{M})\times\ldots \times [\tfrac{m_d}{M},\tfrac{m_d+1}{M})$ indexed by integer multi-indices $\mathbf m=(m_1,\ldots,m_d)\in \mathbb Z^d$. The cube $[0,1]^d$  overlaps with $(M+1)^d$ such cubes $I_{M,\mathbf m}$, namely those with $0\le m_k\le M$. Each of these cubes is mapped by $g_M$ to a unique integer in the range $[1,(M+1)^d].$

Consider the periodic function 
\begin{equation}\label{eq:phi}
\phi(x)=x-\lfloor x\rfloor,\quad \phi:\mathbb R\to [0,1).
\end{equation}
 We will now seek our approximation in the form 
\begin{equation}\label{eq:wfx}
\widetilde f(\mathbf x)=u(g_M(\mathbf x)), u(y)=(B-A)\phi(s\sigma_1(\tfrac{\alpha+\beta}{2}+wy))+A,
\end{equation}
where 
\begin{equation}\label{eq:AB}A=\min_{\mathbf x\in[0,1]^d}f(\mathbf x),\quad B=\max_{\mathbf x\in[0,1]^d}f(\mathbf x),\end{equation}
$\tfrac{\alpha+\beta}{2}$ is the center of the interval $(\alpha,\beta)$ where $\sigma_1$ is analytic and non-polynomial, and  $s$ and $w$ are some weights to be chosen shortly. Clearly, the computation defined by Eqs.~\eqref{eq:gm},\eqref{eq:phi},\eqref{eq:wfx} is representable by a neural network of a fixed size only depending on $d$ (as $O(d)$) and using activations from $\mathcal A_1$. 

Let $N=(M+1)^d.$
Using the uniform continuity of $f$ and choosing $M$ large so that the size of each cube $I_{M,\mathbf m}$ is arbitrarily small, we see that the superexpressiveness will be established if we show that for any $N$, any $\epsilon>0$, and any $\mathbf y\in [A,B]^N$ there exist some weights $s$ and $w$ such that 
\begin{equation}\label{eq:unyn}
|u(n)-y_n|<\epsilon\text{ for all } n=1,\ldots,N.
\end{equation}

Recall that a set of numbers $a_1,\ldots,a_N$ is called \emph{rationally independent} if they are linearly independent over the field $\mathbb Q$ (i.e., no equality $\sum_{n=1}^N \lambda_n a_n=0$ with rational coefficients $\lambda_n$ can hold unless  all $\lambda_n=0$). Our strategy will be:
\begin{enumerate}
\item to choose the weight $w$ so as to make the values $a_n=\sigma_1(\tfrac{\alpha+\beta}{2}+wn)$ with $n=1,\ldots,N$ rationally independent;
\item use the density of the irrational winding on the torus to find $s$ ensuring condition \eqref{eq:unyn}.
\end{enumerate}
  
\noindent
For step 1, we state the following lemma.

\begin{lemma}\label{lm:realan}
Let $\sigma$ be a real analytic function in an interval $(\alpha,\beta)$ with $\beta>\alpha$. Suppose that there is $N$ such that for all $w$ with sufficiently small absolute value, the values $(\sigma(\tfrac{\alpha+\beta}{2}+wn))_{n=1}^N$ are not rationally independent. Then $\sigma$ is a polynomial. 
\end{lemma} 
\begin{proof}
For fixed coefficients $\bm\lambda=(\lambda_1,\ldots,\lambda_N)$, the function $\sigma_{\bm\lambda}(w)=\sum_{n=1}^N {\lambda}_n \sigma(\tfrac{\alpha+\beta}{2}+nw)$ is real analytic for $w\in U_N=(-\tfrac{\beta-\alpha}{2N},\tfrac{\beta-\alpha}{2N})$. Since there are only countably many $\bm\lambda\in\mathbb Q^N$, we see that under hypothesis of the lemma there is some ${\bm\lambda}$ such that $\sigma_{\bm\lambda}$ vanishes on an uncountable subset of $U_N$. Then, by analyticity, $\sigma_{\bm\lambda}\equiv 0$ on $U_N.$ Expanding this $\sigma_{\bm\lambda}$ into the Taylor series at $w=0$, we get the identity $\sum_{n=1}^N {\lambda}_n n^m=0$ for each $m$ such that $\tfrac{d^m \sigma}{dw^m}(\tfrac{\alpha+\beta}{2})\ne 0$. If there are infinitely many such $m$, then all ${\lambda}_n=0$ (by letting $m\to\infty$). It follows that if ${\bm\lambda}$ is nonzero, then there are only finitely many $m$'s such that $\tfrac{d^m \sigma}{dw^m}(\tfrac{\alpha+\beta}{2})\ne 0$, i.e. $\sigma$ is a polynomial.
\end{proof}

Applying Lemma \ref{lm:realan} to $\sigma=\sigma_1$, we see that for any $N$ there is $w$ such that the values $a_n=\sigma_1(\tfrac{\alpha+\beta}{2}+wn)$ with $n=1,\ldots,N$ are rationally independent.

For step 2, we use the well-known fact that an irrational winding on the torus is dense: 
\begin{lemma}\label{lm:irrwind}
Let $a_1,\ldots,a_N$ be rationally independent real numbers. Then the set $Q_N=\{(\phi(sa_1),\ldots,\phi(sa_N)): s\in\mathbb R\}$ (where $\phi$ is defined in Eq.~\eqref{eq:phi}) is dense in $[0,1)^N.$
\end{lemma} 
For completeness, we provide a proof in Appendix~\ref{sec:proofirrwind}. 

Lemma \eqref{lm:irrwind} implies that for any $\mathbf y\in [A,B]^N$, the point $\tfrac{\mathbf y-A}{B-A}\in [0,1]^N$ can be approximated by vectors $(\phi(sa_n))_{n=1}^N$. This implies condition \eqref{eq:unyn}, thus finishing the proof for $\mathcal A_1.$ 

\medskip
\noindent
\emph{Proof for $\mathcal A_2$.} We will only give a proof for $d=1$; the claim then follows for all larger $d$ by Corollary \ref{corol:1}.

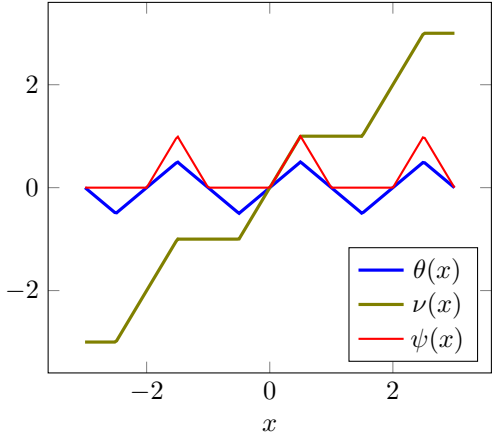
\begin{figure}
\begin{minipage}[b]{.49\textwidth}
\begin{center}
\begin{tikzpicture}
\begin{axis}[scale=0.7, xlabel=$x$,
legend pos=south east
]

\pgfplotsset{
/pgf/declare function={
theta(\x)=1/pi*rad(asin(sin(deg(pi*\x))));
}
}

\pgfplotsset{
/pgf/declare function={
nu(\x)=\x+theta(\x);
}
}

\pgfplotsset{
/pgf/declare function={
psi(\x)=(nu(theta(\x)-0.5)+1);
}
}

\addplot[domain=-3:3, 
         samples=200, 
         color=blue, 
         very thick
         ]{theta(x)};
\addlegendentry{$\theta(x)$}

\addplot[domain=-3:3, 
         samples=200, 
         color=red!50!green!,
         very thick
         ]{nu(x)};
\addlegendentry{$\nu(x)$}
         
\addplot[domain=-3:3, 
         samples=200, 
         color=red,
         thick]{psi(x)};
\addlegendentry{$\psi(x)$}         
\end{axis}
\end{tikzpicture}
\caption{The functions $\theta, \nu, \psi$ from the proof of Theorem~\ref{th:main}.}\label{fig:theta}
\end{center}
\end{minipage}
\end{figure}

Consider the piecewise linear periodic function 
$$\theta(x)=\tfrac{1}{\pi}\arcsin(\sin \pi x)$$ 
and the related functions 
\begin{align*}
\nu(x)={}&x+\theta(x),\\
\psi(x)={}&\nu(\theta(x)-\tfrac{1}{2})+1
\end{align*}
(see Fig.~\ref{fig:theta}).

We would like to extend the previous proof for $\mathcal A_1$ to the present case of $\mathcal A_2$ using the function $\nu$ as a substitute for $\lfloor\cdot\rfloor,$ since $\nu$ is constant on the intervals $[k-\tfrac{1}{2},k+\tfrac{1}{2}]$ with odd integer $k$.  However, in contrast to the function $\lfloor\cdot\rfloor$, the function $\nu$ is continuous and cannot map the whole segment $[0,1]$ to a finite set of values, which was crucial in the proof for $\mathcal A_1$. For this reason, we use a partition of unity and represent the approximated function $f\in C([0,1])$ as a sum of four functions supported on a grid of disjoint small segments. Specifically, let again $M$ be a large integer determining the scale of our partition of unity. We define this partition by 
\begin{equation}\label{eq:psiq}
1\equiv\sum_{q=-1}^2\psi_q(x),\quad \psi_q(x)=\psi(Mx-\tfrac{q}{2}),\quad x\in\mathbb R,
\end{equation}
and the respective decomposition of the function $f$ by
\begin{equation}\label{eq:fsumq}
f=\sum_{q=-1}^2f_{q},\quad f_{q}=f\psi_{q}.
\end{equation}
For a fixed $q$, the function $\psi_{q}$, and hence also $f_{q}$, vanish outside of the union of $N=\tfrac{M}{2}+O(1)$ disjoint segments $J_{q, p}=[\tfrac{4p-2+q}{2M}, \tfrac{4p+q}{2M}],p=1,\ldots,N,$ overlapping with the segment $[0,1].$ Denote this union by $ J_{q}.$

We approximate each function $f_{q}$ by a function $\widetilde f_{q}$ using an analog of the representation \eqref{eq:gm},\eqref{eq:phi},\eqref{eq:wfx}:
\begin{align}
G_{q}(x)={}&\nu(Mx-\tfrac{q}{2}+\tfrac{1}{2}),\label{eq:Gq}\\
v_{ q}(x)={}&(2\max_{x\in[0,1]}|f(x)|)\theta(s\sin(wG_{q}(x))),\label{eq:vq}\\
\widetilde f_{q}(x)={}&v_{q}(x)\psi_{q}(x),\label{eq:wfqx0}\\
\widetilde f={}&\sum_{q=-1}^2 \widetilde f_q.\label{eq:wffq}
\end{align} 

The function $G_{q}$ in Eq.~\eqref{eq:Gq} is constant and equal to $2p-1$ on each segment $J_{q, p}$. In particular, different segments $J_{q, p}$ overlapping with the segment $[0,1]$ are mapped by $G_{q}$ to different integers in the interval $[1,M+1].$

The function $v_{q}$ in Eq.~\eqref{eq:vq} is the analog of the expression for $\widetilde f$ given in Eq.~\eqref{eq:wfx}. Like $G_q$, the function $v_q$ is constant on each interval $J_{q,p}$. By Lemma \ref{lm:realan}, the values $(\sin (wm))_{m=1}^{M+1}$ are rationally independent for a suitable $w.$ We can then use again the density of irrational winding on the torus (Lemma~\ref{lm:irrwind}) to find $s$ such that for each $p$ the value $v_q|_{J_{q,p}}$ is arbitrarily close to the value of $f$ at the center $x_{q,p}=\tfrac{4p-1+q}{2M}$ of the interval $J_{q,p}$. Indeed, $\theta$ is a continuous periodic (period--2) function with $\max_x \theta(x)=-\min_x\theta(x)=\tfrac{1}{2}$.  For each $p=1,\ldots,N,$ we can first find $z_{q,p}\in\mathbb R/(2\mathbb Z)$ such that $(2\max_{x\in[0,1]}|f(x)|)\theta(z_{p,q})=f(x_{q,p}),$ and then, by Lemma \ref{lm:irrwind}, find $s$ such that $s\sin(w(2p-1))$ is arbitrarily close to $z_{q,p}$ on the circle $\mathbb R/(2\mathbb Z)$ for each $p=1,\ldots,N.$ As a result, we see that the function $v_{q}$ can approximate the function $f$ on the whole set $J_{q}$. As before, to achieve  an arbitrarily small error, we need to first choose $M$ large enough and then choose suitable $w$ and $s$. (By the uniform continuity of $f$, one can use here the same $w$ and $s$ for all $q\in\{-1,0,1,2\}$.)

At the same time, it makes no difference how the function $v_{q}$ behaves on the complementary set $[0,1]\setminus J_{q}$, since $\psi_{q}$ vanishes on this set. It follows that $\widetilde f_{q}$ defined by Eq.~\eqref{eq:wfqx0} can approximate $f_q$ defined by Eq.~\eqref{eq:fsumq} with arbitrarily small error on the whole segment $[0,1].$ 
Then, the function $\widetilde f$ given by Eq.~\eqref{eq:wffq} can approximate $f$ uniformly on $[0,1]$ with any accuracy. 

The computation \eqref{eq:Gq}-\eqref{eq:wffq} is directly representable by a fixed size neural network with activations $\{\sin,\arcsin\}$, except for multiplication step \eqref{eq:wfqx0}. Multiplication, however, can be implemented with any accuracy by a fixed-size subnetwork:

\begin{lemma}[Approximate multiplier]\label{lm:mult}
Suppose that an activation function $\sigma$ has a point $x_0$ where the second derivative $\tfrac{d^2\sigma}{dx^2}(x_0)$ exists and is nonzero. Then there is a fixed two-input network architecture with this activation that allows to implement the approximate multiplication of the inputs, $x,y\mapsto xy,$ with any accuracy  uniformly on any bounded set of inputs $x,y$, by suitably adjusting the weights.  
\end{lemma}
\begin{proof}
First note that we can implement the approximate squaring $x\mapsto x^2$ with any accuracy using just a network with three neurons. Indeed, by the assumption on $\tfrac{d^2\sigma}{dx^2}$, for any $C,\epsilon>0$ we can choose $\delta$ such that 
$$|(\tfrac{d^2\sigma}{dx^2}(x_0))^{-1}\tfrac{1}{\delta^2}(\sigma(x_0+x\delta)+\sigma(x_0-x\delta)-2\sigma(x_0))-x^2|<\epsilon$$
for all $|x|<C$. Then, using
the polarization identity $xy=\tfrac{1}{2}((x+y)^2+(x-y)^2)$, we see that the desired approximate multiplier can be implemented using a fixed 6-neuron architecture. 
\end{proof}

We can apply this lemma with $\sigma=\sin$ and any $x_0\ne \pi k,k\in\mathbb Z$, thus completing the proof for $\mathcal A_2$.

\medskip
\noindent
\emph{Proof for $\mathcal A_3$.}
We reduce this case to the previous one, $\mathcal A_2.$ First observe that we can approximate the function $\arcsin$ by a fixed-size $\sigma_3$-network. 
\begin{lemma}\label{lm:antider}
A superexpressive family of continuous activations remains superexpressive if some activations are replaced by their antiderivatives.
\end{lemma}
\begin{proof}
The claim follows since any continuous activation $\sigma$ can be approximated uniformly on compact sets by expressions $\tfrac{1}{\delta}(\sigma^{(-1)}(x+\delta)-\sigma^{(-1)}(x))$, where $\sigma^{(-1)}=\int \sigma$.
\end{proof}
Our activation $\sigma_3$ is the antiderivative of  $\tfrac{1}{\pi}\arcsin x+\tfrac{3}{2}$ on the interval $[-1,1]$. 

Observe next that on the interval $[1,\infty),$ we can express the function $\sin x$ by multiplying $\sigma_3(x)$ by some polynomials in $x$ and subtracting constants. By Lemma \ref{lm:mult}, these operations can be implemented with any accuracy by a fixed size $\sigma_3$-network. By periodicity of $\sin,$ we can then approximate it on any bounded interval.

We conclude that we can approximate any $\mathcal A_2$-network with any accuracy by a $\sigma_3$-network that has the same size up to a constant factor.  

It is an elementary computation that $\sigma_3$ is $ C^1(\mathbb R)$, bounded and monotone increasing. This completes the proof of the theorem.
\end{proof}

\section{Absence of superexpressiveness for standard activations}\label{sec:absence}
In this section we show that most practically used activation functions (those not involving $\sin x$ or $\cos x$) are not superexpressive. This is an easy consequence of Khovanskii's bounds on the number of zeros of elementary functions \cite{khovanskii}. We remark that these bounds have been used previously to bound expressiveness of neural networks in terms of VC dimension \cite{karpinski1997polynomial} or Betti  numbers of level sets \cite{bianchini2014complexity}. 

First recall the standard definition of Pfaffian functions (see e.g. \cite{khovanskii, zell1999betti, gabrielov2004complexity}). 
A \emph{Pfaffian chain} is a sequence $f_1,\ldots,f_l$ of real analytic functions defined on a common connected domain $U\subset\mathbb R^d$ and such that the equations
$$\tfrac{\partial f_i}{\partial x_j}(\mathbf x)=P_{ij}(\mathbf x,f_1(\mathbf x),\ldots,f_i(\mathbf x)), 1\le i\le l,\; 1\le j\le d$$
hold in $U$ for some polynomials $P_{ij}.$  
 A \emph{Pfaffian function} in the chain $(f_1,\ldots,f_l)$ is a function on $U$ that can be expressed as a polynomial $P$ in the variables $(\mathbf x, f_1(\mathbf x),\ldots,f_l(\mathbf x))$. 
\emph{Complexity} of the Pfaffian function $f$ is the triplet $(l,\alpha,\beta)$ consisting of the length $l$ of the chain, the maximum degree $\alpha$ of the polynomials $P_{ij}$, and the degree $\beta$ of the polynomial $P.$

The importance of Pfaffian functions stems from the fact that they include all elementary functions when considered on suitable domains. This is shown by first checking that the simplest elementary functions are Pfaffian, and then by checking that arithmetic operations and compositions of Pfaffian functions produce again Pfaffian functions. We refer again to \cite{khovanskii, zell1999betti, gabrielov2004complexity} for details.
 
\begin{prop} \hfill
\begin{enumerate}
\item (Elementary examples) The following functions are Pfaffian: polynomials on $U=\mathbb R^d,$  $e^x$ on $\mathbb R$,  $\ln x$ on $\mathbb R_+$,  $\arcsin x$ on $(-1,1).$
 The function $\sin x$ is Pfaffian on any bounded interval $(A,B),$ with complexity depending on $B-A$, but $\sin x$ is \emph{not} Pfaffian on $\mathbb R$. 
\item (Operations with Pfaffian functions) Sums and products of Pfaffian functions $f,g$ with a common domain $U$ are Pfaffian. If the domain of a Pfaffian function $f$ includes the range of a Pfaffian function $g$, then the composition $f\circ g$ is Pfaffian on the domain of $g$. The complexity of the resulting functions $f+g,fg,f\circ g$ is determined by the complexity of the functions $f,g$. 
\end{enumerate}
\end{prop}

We state now the fundamental result on Pfaffian functions. We call a solution $\mathbf x\in\mathbb R^d$ of a system $f_1(\mathbf x)=\ldots=f_d(\mathbf x)=0$ \emph{nondegenerate} if the respective Jacobi matrix $\tfrac{\partial f_i}{\partial x_j}(\mathbf x)$ is nondegenerate.
 
\begin{theor}[\citealt{khovanskii}]\label{th:khovanskii} Let $f_1,\ldots,f_d$ be Pfaffian $d$-variable functions with a common Pfaffian chain on a connected domain $U$. Then the number of nondegenerate solutions of the system $f_1(\mathbf x)=\ldots=f_d(\mathbf x)=0$ is bounded by a finite number only depending on the complexities of the functions $f_1,\ldots,f_d.$  
\end{theor}
The idea of the proof is to use a generalized Rolle's lemma and bound the number of common zeros of the functions $f_k$ by the number of common zeros  of suitable polynomials (in a larger number of variables). The latter number can then be upper bounded using the classical B\'ezout theorem. It is possible to write the bound in Theorem \ref{th:khovanskii} explicitly, but we will not need that for our purposes. 

We will only use the univariate version of Theorem \ref{th:khovanskii}. In this case, it will also be easy to remove the inconvenient nondegeneracy condition in this theorem. (Note that this condition is essential in  general -- for example, if  $f_1(x_1,x_2)\equiv f_2(x_1,x_2)=x_1$, then the system $f_1=f_2=0$ has infinitely many degenerate solutions).
 
\begin{prop}\label{prop:univpfaff} Let $f$ be a univariate Pfaffian function on an open interval $I\subset\mathbb R.$ Then either $f\equiv 0$ on $I$, or the number of zeros of $f$ is bounded by a finite number only depending on the complexity of $f$. 
\begin{proof}
Suppose $f\not\equiv 0$. Then, by real analiticity of $f$, any zero $x_0$ of $f$ in $I$ is isolated, and we can write $f(x)=c(x-x_0)^k(1+o(1))$ as $x\to x_0$, with some $c\ne 0$ and $k\in\mathbb N.$ By Sard's theorem, there is a sequence $\epsilon_n\searrow 0$ such that the values $\pm\epsilon_n$ are not critical values of $f$. The functions $f\pm \epsilon_n$ are Pfaffian with the same complexity as $f$, and don't have degenerate zeros. For any zero $x_0$ of $f$, the two functions $f\pm\epsilon_n$ will have in total two nondegenerate zeros in a vicinity of $x_0,$ for any $\epsilon_n$ small enough. It follows that the total number of all nondegenerate zeros of the two functions $f\pm\epsilon_n$, for $\epsilon_n$ small enough, will be at least twice as large as the number or zeros of the function $f$ (or can be made arbitrarily large if $f$ has infinitely many zeros). Applying Theorem \ref{th:khovanskii} to the functions $f\pm\epsilon_n$, we obtain the desired conclusion on the zeros of $f$.
\end{proof}
\end{prop}

Now we apply these results to standard activation functions.

\begin{defin} We say that an activation function $\sigma$ is \emph{piecewise Pfaffian} if its domain of definition can be represented as a union of finitely many open intervals $U_n$ and points $x_k$ in $\mathbb R$ so that $\sigma$ is Pfaffian on each $U_n$.  
\end{defin}

By discussion above, this definition covers most practically used activations, such as $\tanh x$, standard sigmoid $\sigma(x)=(1+e^{-x})^{-1},$ ReLU $\sigma(x)=\max(0,x),$ leaky ReLU $\sigma(x)=\max(ax,x),$ binary step function $\sigma(x)=\begin{cases}0,&x<0\\1,&x\ge 0\end{cases}$, Gaussian $\sigma(x)=e^{-x^2},$ softplus $\sigma(x)=\ln(1+e^x)$ \cite{glorot2011deep}, ELU $\sigma(x)=\begin{cases}a(e^x-1),&x<0\\ x,& x\ge 0\end{cases}$ \cite{clevert2015fast}, etc. Our main result in this section states that any finite collection of such activations is not superexpressive.

\begin{theor}\label{th:main2}
Let $\mathcal A$ be a family of finitely many piecewise Pfaffian activation functions. Then $\mathcal A$ is not superexpressive.
\end{theor}
\begin{proof}
Suppose that $\mathcal A$ is superexpressive, and there is a fixed one-input network architecture allowing us to approximate any univariate function $f\in C([0,1])$. Then for any $N$ we can choose the network weights so that the  function $\widetilde f$ implemented by the network has at least $N$ sign changes, in the sense that there are points $0\le a_0<\ldots<a_N\le 1$ such that $(-1)^n\widetilde f(a_n)>0$ for all $n$.  Indeed, this follows simply by approximating the function $f(x)=\sin((N+1)\pi x)$ with an error less than 1. We will show, however, that this $N$ cannot be arbitrarily large if the activations are from a finite piecewise Pfaffian family. 

\begin{lemma}
If the activations belong to a finite piecewise Pfaffian family $\mathcal A$, then any function $\widetilde f$ implemented by the network is piecewise Pfaffian. Moreover, the number of respective intervals $U_n$ as well as the complexity of each restriction $f|_{U_n}$ do not exceed some finite values only depending on the family $\mathcal A$ and the network architecture.
\end{lemma} 
\begin{proof} This can be proved by induction on the number of hidden neurons in the network. The base of induction corresponds to networks without hidden neurons; in this case the statement is trivial. Now we make the induction step. Given a network, choose some hidden neuron whose output is not used by other hidden neurons (i.e., choose a neuron in the ``last hidden layer''). With respect to this neuron, we can decompose the network output as 
\begin{equation}\label{eq:wtfpfaf}
\widetilde f(x)=c\sigma_k\Big(\sum_{s=1}^Kc_s\widetilde f_{s}(x)+h\Big)+\sum_{s=1}^K c'_s\widetilde f_{s}(x)+h'.
\end{equation}
Here, $\sigma_k$ is the activation function residing at the chosen neuron, $\widetilde f_s$ are the signals going out of the other hidden and input neurons, and  $c,c_s,c_s',h,h'$ are various weights.  By inductive hypothesis, all functions $\widetilde f_s$ here are piecewise Pfaffian. Moreover, by taking intersections, the segment $[0,1]$ can be divided into finitely many open intervals $I_j$ separated by finitely many points $x_l$ so that each of the functions $\widetilde f_s$ is Pfaffian on each interval $I_j.$ The number of these intervals $I_j$ and the complexities of $f_s|_{I_j}$ are bounded by some finite values depending only on the family $\mathcal A$ and the network architecture. We see also that the linear combination $F(x) = \sum_{s=1}^Kc_s\widetilde f_{s}(x)+h$ appearing in Eq.~\eqref{eq:wtfpfaf} is Pfaffian on each interval $I_j.$ 

Observe next that the composition $\sigma_k\circ F$ is  piecewise Pfaffian on each interval $I_j$. Indeed, let $U^{(k)}_r$ and $x^{(k)}_r$ be the finitely many open intervals and points associated with the activation $\sigma_k$ as a piecewise Pfaffian function. By Proposition \ref{prop:univpfaff}, for each $r$, the pre-image $(F|_{I_j})^{-1}(x_r^{(k)})$ is either the whole interval $I_j$ or its finite subset. In the first case, $\sigma_k\circ F$ is constant and thus trivially Pfaffian on $I_j$. In the second case, the interval $I_j$ can be subdivided into sub-intervals $I_{j,m}$ such that each image $F(I_{j,m})$ belongs to one of the intervals $U^{(k)}_r$ so that $\sigma_k\circ F$ is Pfaffian on $I_{j,m}$. The number of these sub-intervals and the complexities of the restrictions are bounded by some finite numbers depending on the activation $\sigma_k$ and the complexity of $F|_{I_j}.$  

Returning to representation \eqref{eq:wtfpfaf}, we see that $\widetilde f$ is Pfaffian on each interval $I_{j,m}$; moreover, the total number of these intervals as well as the complexities of the restrictions $\widetilde f|_{I_{j,m}}$ are bounded by finite numbers determined by the family $\mathcal A$ and the architecture, thus proving the claim. 
\end{proof}   
The lemma implies that some interval $U_n$ in which $\widetilde f$ is Pfaffian and has a bounded complexity can contain an arbitrarily large number of sign changes of $\widetilde f$. This gives a contradiction with  Proposition \ref{prop:univpfaff}.
\end{proof}

\section{Discussion}
We have given examples of simple explicit activation functions that allow to approximate arbitrary functions using fixed-size networks (Theorem \ref{th:main}), and we have also shown that this can not be achieved with the common practical activations (Theorem \ref{th:main2}). We mention two interesting questions left open by our results. 

First, our existence result (Theorem \ref{th:main}) is of course purely theoretical: though the network is small, a huge approximation complexity is hidden in the very special choice of the network weights.  Nevertheless, assuming that we can perform computations with any precision, one can ask if it is possible to algorithmically find network weights providing a good approximation. The main difficulty here is to find a value $s$ such that $(\phi(sa_n))_{n=1}^N$  is close to the given $N$-dimensional point. Such a value exists by  Lemma \ref{lm:irrwind} on the density of irrational winding, and the proof of the lemma is essentially constructive, so theoretically one can perform the necessary computation and find the desired $s$.  However, the proof is based on the pigeonhole principle and is very prone to the curse of dimensionality (with dimensionality here corresponding to the number $N$ of fitted data points), making this computation practically unfeasible even for relatively small $N$.

Another open question is whether the function $\sin$ alone is superexpresive. This can not be ruled out by the methods of Section \ref{sec:absence}, since $\sin$ has an infinite Pfaffian complexity on $\mathbb R$. More generally, one can ask if there are individual superexpressive activations that are elementary and real analytic on the whole $\mathbb R$. A repeated computation of antiderivatives using  Lemma \ref{lm:antider} allows us to construct a piecewise elementary superexpressive function of any finite smoothness, but not analytic on $\mathbb R$.

\appendix
\section{Proof of Lemma \ref{lm:irrwind}}\label{sec:proofirrwind}
It is convenient to endow the cube $[0,1)^N$ with the topology of the torus $\mathbb T^N=\mathbb R^N/\mathbb Z^N$ by gluing the endpoints of the interval $[0,1]$. Though the lemma is stated in terms of the original topology on $[0,1)^N$, it is clear that a subset is dense in the original topology if and only if it is dense in the topology of the torus. Accordingly, when considering the distance between two points $\mathbf b_1,\mathbf b_2\in [0,1)^N,$ it will be convenient to  use the distance between the corresponding cosets, i.e. $\rho(\mathbf b_1,\mathbf b_2)=\min_{\mathbf z_1,\mathbf z_2\in \mathbb Z^N}|\mathbf b_1+\mathbf z_1-(\mathbf b_2+\mathbf z_2)|,$ where $|\cdot |$ is the usual euclidean norm. Note that this $\rho$ is a shift--invariant metric on the torus.

The proof of the lemma is by induction on $N$. The base $N=1$ is obvious (a single number $a_1$ is rationally independent iff $a_1\ne 0$). 
Let us make the induction step from $N-1$ to $N$, with $N\ge 2$. 

Given the rationally independent numbers $a_1,\ldots,a_N,$ first observe that none of them equals 0. Let $s_0=\tfrac{1}{a_N}$. Let $\phi(x)=x-\lfloor x\rfloor$ as in Eq.~\eqref{eq:phi}. If $s=ms_0$ with some integer $m$, then $\phi(ms_0 a_N)=0,$ so that the points $\mathbf b_m=(\phi(ms_0 a_1),\ldots,\phi(m s_0 a_N))$ lie in the $(N-1)$-dimensional face $[0,1)^{N-1}$ of the full set $[0,1)^N.$ 

Observe that the points $\mathbf b_m$ are different for different integer $m$'s. Indeed, if $\mathbf b_{m_1}=\mathbf b_{m_2}$ for some integer $m_1\ne m_2,$ then there are some  integers $p_1,\ldots,p_N$ such that $(m_1-m_2)s_0a_n=p_n$ for all $n=1,\ldots,N.$ But then the numbers $a_1,\ldots,a_N$ are not rationally independent, since, e.g., $(m_1-m_2)a_1=\tfrac{p_1}{s_0}=p_1a_N$.

Since  the points $\mathbf b_m$ are distinct, they form an infinite set in $[0,1)^{N-1}$. Then for any $\epsilon$ we can find a pair of different points $\mathbf b_{m_1}$ and $\mathbf b_{m_2}$ separated by a distance $\rho(\mathbf b_{m_1},\mathbf b_{m_2})<\epsilon.$ Note that the distance $\rho(\mathbf b_{m_1},\mathbf b_{m_2})$ only depends on the difference $m_2-m_1$, so we can assume that $m_1=0$:
\begin{equation*}
\rho(\mathbf b_{0},\mathbf b_{m_2})=\rho(\mathbf 0,\mathbf b_{m_2})<\epsilon.
\end{equation*} 
By definition of $\rho$, we can then find $\mathbf z\in\mathbb Z^N$ such that for $\mathbf b_{m_2}'=\mathbf b_{m_2}-\mathbf z$ we have
\begin{equation}\label{eq:rhoeps}|\mathbf b_{m_2}'|=\rho(\mathbf 0,\mathbf b_{m_2})<\epsilon.\end{equation}

We can write $\mathbf b_{m_2}'$ in the form $$\mathbf b_{m_2}'=(m_2s_0a_{1}-p_{1},\ldots,m_2s_0a_{N-1}-p_{N-1},0)$$ with some integers $p_{1},\ldots,p_{N-1}$.  Observe that the first $N-1$ components $b_{m_2,n}'$ of $\mathbf b_{m_2}'$ are rationally independent. Indeed, if $\sum_{n=1}^{N-1}\lambda_nb_{m_2,n}'=0$ with some rational $\lambda_n$, then, by expressing this identity in terms of the original values $a_n$, we get $$\sum_{n=1}^{N-1}\lambda_na_n-\tfrac{1}{m_2}\sum_{n=1}^{N-1}\lambda_n p_{n}a_N=0,$$ so $\lambda_n\equiv 0$ by the rational independence of $a_n.$

Consider now the set  $$Q'_{N-1}=\{\phi(tb_{m_2,1}'),\ldots,\phi(tb_{m_2,N-1}')): t\in\mathbb R\}.$$ On the one hand, by induction hypothesis, the set $Q'_{N-1}$ is dense in $[0,1)^{N-1},$ because the numbers $b_{m_2,n}'$ are rationally independent. 
On the other hand, observe that the points in $Q'_{N-1}$ corresponding to integer $t$ also belong to the set $Q_N=\{(\phi(sa_1),\ldots,\phi(sa_N)): s\in\mathbb R\}$: specifically, the respective $s=m_2ts_0.$ It follows that for any $\mathbf b\in [0,1)^{N-1}$ we can find a point $\widetilde{\mathbf b}$ of the set $Q_N$ at a distance at most $2\epsilon$ from $\mathbf b$:      
first find a point $\widehat{\mathbf b}\in Q'_{N-1}$ such that $|\widehat{\mathbf b}-\mathbf b|<\epsilon$, and then, if $\widehat{\mathbf b}$ corresponds to some $t=t_0$ in $Q'_{N-1}$, take $\widetilde{\mathbf b}$ corresponding to $t=\lfloor t_0\rfloor.$ The distance $\rho(\widehat{\mathbf b},\widetilde{\mathbf b})<\epsilon$ by Eq.~\eqref{eq:rhoeps} and because $|t_0-\lfloor t_0\rfloor|< 1$:
$$\rho(\widehat{\mathbf b},\widetilde{\mathbf b})\le |t_0\mathbf b_{m_2}'-\lfloor t_0\rfloor\mathbf b_{m_2}'|<|\mathbf b_{m_2}'|<\epsilon.$$

The above argument shows that the face $[0,1)^{N-1}=\{\mathbf b\in [0,1)^N: b_N=0\}$ can be approximated by points of $Q_N$ with $s$ belonging to the set $S=\{m_2ts_0\}_{t\in\mathbb Z}.$ Any other $(N-1)$-dimensional cross-section $\{\mathbf b\in [0,1)^N: b_N=c\}$ is then approximated by the points of $Q_N$ with $s\in S+cs_0$: indeed, $s=cs_0$ gives us one point  in this cross-section, and additional shifts by $\Delta s\in S$ allow us to approximate any other point with the same $b_N$.  

\section*{Acknowledgment}
I thank Maksim Velikanov for useful feedback on the preliminary version of the paper.

\bibliography{superexpressive}

\begin{thebibliography}{22}
\providecommand{\natexlab}[1]{#1}
\providecommand{\url}[1]{\texttt{#1}}
\expandafter\ifx\csname urlstyle\endcsname\relax
  \providecommand{\doi}[1]{doi: #1}\else
  \providecommand{\doi}{doi: \begingroup \urlstyle{rm}\Url}\fi

\bibitem[Bianchini \& Scarselli(2014)Bianchini and
  Scarselli]{bianchini2014complexity}
Bianchini, M. and Scarselli, F.
\newblock On the complexity of neural network classifiers: A comparison between
  shallow and deep architectures.
\newblock \emph{IEEE transactions on neural networks and learning systems},
  25\penalty0 (8):\penalty0 1553--1565, 2014.

\bibitem[Clevert et~al.(2015)Clevert, Unterthiner, and
  Hochreiter]{clevert2015fast}
Clevert, D.-A., Unterthiner, T., and Hochreiter, S.
\newblock Fast and accurate deep network learning by exponential linear units
  (elus).
\newblock \emph{arXiv preprint arXiv:1511.07289}, 2015.

\bibitem[DeVore et~al.(1989)DeVore, Howard, and Micchelli]{devore1989optimal}
DeVore, R.~A., Howard, R., and Micchelli, C.
\newblock Optimal nonlinear approximation.
\newblock \emph{Manuscripta mathematica}, 63\penalty0 (4):\penalty0 469--478,
  1989.

\bibitem[Gabrielov \& Vorobjov(2004)Gabrielov and
  Vorobjov]{gabrielov2004complexity}
Gabrielov, A. and Vorobjov, N.
\newblock {Complexity of computations with Pfaffian and Noetherian functions}.
\newblock \emph{Normal forms, bifurcations and finiteness problems in
  differential equations}, 137:\penalty0 211--250, 2004.

\bibitem[Glorot et~al.(2011)Glorot, Bordes, and Bengio]{glorot2011deep}
Glorot, X., Bordes, A., and Bengio, Y.
\newblock Deep sparse rectifier neural networks.
\newblock In \emph{Proceedings of the fourteenth international conference on
  artificial intelligence and statistics}, pp.\  315--323, 2011.

\bibitem[Guliyev \& Ismailov(2016)Guliyev and Ismailov]{guliyev2016single}
Guliyev, N.~J. and Ismailov, V.~E.
\newblock A single hidden layer feedforward network with only one neuron in the
  hidden layer can approximate any univariate function.
\newblock \emph{Neural computation}, 28\penalty0 (7):\penalty0 1289--1304,
  2016.

\bibitem[Guliyev \& Ismailov(2018{\natexlab{a}})Guliyev and
  Ismailov]{guliyev2018approximation1}
Guliyev, N.~J. and Ismailov, V.~E.
\newblock Approximation capability of two hidden layer feedforward neural
  networks with fixed weights.
\newblock \emph{Neurocomputing}, 316:\penalty0 262--269, 2018{\natexlab{a}}.

\bibitem[Guliyev \& Ismailov(2018{\natexlab{b}})Guliyev and
  Ismailov]{guliyev2018approximation2}
Guliyev, N.~J. and Ismailov, V.~E.
\newblock On the approximation by single hidden layer feedforward neural
  networks with fixed weights.
\newblock \emph{Neural Networks}, 98:\penalty0 296--304, 2018{\natexlab{b}}.

\bibitem[Igelnik \& Parikh(2003)Igelnik and Parikh]{igelnik2003kolmogorov}
Igelnik, B. and Parikh, N.
\newblock Kolmogorov's spline network.
\newblock \emph{IEEE transactions on neural networks}, 14\penalty0
  (4):\penalty0 725--733, 2003.

\bibitem[Ismailov(2014)]{ismailov2014approximation}
Ismailov, V.~E.
\newblock On the approximation by neural networks with bounded number of
  neurons in hidden layers.
\newblock \emph{Journal of Mathematical Analysis and Applications},
  417\penalty0 (2):\penalty0 963--969, 2014.

\bibitem[Karpinski \& Macintyre(1997)Karpinski and
  Macintyre]{karpinski1997polynomial}
Karpinski, M. and Macintyre, A.
\newblock {Polynomial bounds for VC dimension of sigmoidal and general Pfaffian
  neural networks}.
\newblock \emph{Journal of Computer and System Sciences}, 54\penalty0
  (1):\penalty0 169--176, 1997.

\bibitem[Khovanskii(1991)]{khovanskii}
Khovanskii, A.~G.
\newblock \emph{{Fewnomials}}.
\newblock Vol. 88 of Translations of Mathematical Monographs. American
  Mathematical Society, 1991.

\bibitem[Kolmogorov(1957)]{kolmogorov1957representation}
Kolmogorov, A.~N.
\newblock On the representation of continuous functions of many variables by
  superposition of continuous functions of one variable and addition.
\newblock In \emph{Doklady Akademii Nauk}, volume 114, pp.\  953--956. Russian
  Academy of Sciences, 1957.

\bibitem[K\r{u}rkov{\'a}(1991)]{kuurkova1991kolmogorov}
K\r{u}rkov{\'a}, V.
\newblock Kolmogorov's theorem is relevant.
\newblock \emph{Neural computation}, 3\penalty0 (4):\penalty0 617--622, 1991.

\bibitem[K\r{u}rkov{\'a}(1992)]{kuurkova1992kolmogorov}
K\r{u}rkov{\'a}, V.
\newblock Kolmogorov's theorem and multilayer neural networks.
\newblock \emph{Neural networks}, 5\penalty0 (3):\penalty0 501--506, 1992.

\bibitem[Maiorov \& Pinkus(1999)Maiorov and Pinkus]{maiorov1999lower}
Maiorov, V. and Pinkus, A.
\newblock Lower bounds for approximation by mlp neural networks.
\newblock \emph{Neurocomputing}, 25\penalty0 (1-3):\penalty0 81--91, 1999.

\bibitem[Montanelli \& Yang(2020)Montanelli and Yang]{montanelli2020error}
Montanelli, H. and Yang, H.
\newblock {Error bounds for deep ReLU networks using the Kolmogorov--Arnold
  superposition theorem}.
\newblock \emph{Neural Networks}, 129:\penalty0 1--6, 2020.

\bibitem[Schmidt-Hieber(2020)]{schmidt2020kolmogorov}
Schmidt-Hieber, J.
\newblock {The Kolmogorov-Arnold representation theorem revisited}.
\newblock \emph{arXiv preprint arXiv:2007.15884}, 2020.

\bibitem[Shen et~al.(2020{\natexlab{a}})Shen, Yang, and Zhang]{shen2020deep}
Shen, Z., Yang, H., and Zhang, S.
\newblock Deep network approximation with discrepancy being reciprocal of width
  to power of depth.
\newblock \emph{arXiv preprint arXiv:2006.12231}, 2020{\natexlab{a}}.

\bibitem[Shen et~al.(2020{\natexlab{b}})Shen, Yang, and Zhang]{shen2020neural}
Shen, Z., Yang, H., and Zhang, S.
\newblock Neural network approximation: Three hidden layers are enough.
\newblock \emph{arXiv preprint arXiv:2010.14075}, 2020{\natexlab{b}}.

\bibitem[Yarotsky \& Zhevnerchuk(2019)Yarotsky and
  Zhevnerchuk]{yarotsky2019phase}
Yarotsky, D. and Zhevnerchuk, A.
\newblock The phase diagram of approximation rates for deep neural networks.
\newblock \emph{arXiv preprint arXiv:1906.09477}, 2019.

\bibitem[Zell(1999)]{zell1999betti}
Zell, T.
\newblock {Betti numbers of semi-Pfaffian sets}.
\newblock \emph{Journal of Pure and Applied Algebra}, 139\penalty0
  (1-3):\penalty0 323--338, 1999.

\end{thebibliography}
\bibliographystyle{icml2021}

\end{document}